\documentclass[letterpaper]{article} 
\usepackage{aaai23}  
\usepackage{times}  
\usepackage{helvet}  
\usepackage{courier}  
\usepackage[hyphens]{url}  
\usepackage{graphicx} 
\usepackage{natbib}  
\usepackage{caption} 
\usepackage{algorithm}
\usepackage{algorithmic}

\usepackage[utf8]{inputenc} 
\usepackage[T1]{fontenc}    
\usepackage{hyperref}       
\usepackage{url}            
\usepackage{booktabs}       
\usepackage{amsfonts}       
\usepackage{nicefrac}       
\usepackage{microtype}      
\usepackage{xcolor}         
\usepackage{multirow}
\usepackage{algorithm}
\usepackage{algorithmic}
\usepackage{bbm}
\usepackage{graphicx}
\usepackage{comment}

\usepackage{amsmath,amsthm,amssymb,multirow,paralist,mathrsfs,amsfonts,dsfont}
\newtheorem{theorem}{Theorem}

\newtheorem{lemma}{Lemma}

\newtheorem{definition}{Definition}
\newtheorem{assumption}{Assumption}

\newtheorem{remark}{Remark}

\usepackage{enumitem}

\def\indicator{{\bf 1}}

\def\CP{\text{CP}}

\def\NCP{\text{NCP}}
\def\cal{\text{cal}}
\def\tr{\text{tr}}
\def\test{\text{test}}

\def\and{\mathrm{and}}
\def\class{\mathrm{class}}

\def\calX{\mathcal X}
\def\calY{\mathcal Y}
\def\calZ{\mathcal Z}
\def\calN{\mathcal N}
\def\calB{\mathcal B}
\def\calR{\mathcal R}

\def\calD{\mathcal D}

\def\E{\mathbb E}
\def\P{\mathbb P}
\def\R{\mathbb R}

\usepackage{newfloat}
\usepackage{listings}
\DeclareCaptionStyle{ruled}{labelfont=normalfont,labelsep=colon,strut=off} 
\floatstyle{ruled}
\newfloat{listing}{tb}{lst}{}
\floatname{listing}{Listing}
\title{
Improving Uncertainty Quantification of Deep Classifiers via Neighborhood Conformal Prediction: Novel Algorithm and Theoretical Analysis
}
\author{
    Subhankar Ghosh\equalcontrib,
    Taha Belkhouja\equalcontrib,
    Yan Yan,
    Janardhan Rao Doppa
}
\affiliations{
School of EECS, Washington State University\\
\{subhankar.ghosh, taha.belkhouja, yan.yan1, jana.doppa\}@wsu.edu
}

\usepackage{bibentry}

\begin{document}

\maketitle

\begin{abstract}
Safe deployment of deep neural networks in high-stake real-world applications requires theoretically sound uncertainty quantification. Conformal prediction (CP) is a principled framework for uncertainty quantification of deep models in the form of prediction set for classification tasks with a user-specified  coverage (i.e., true class label is contained with high probability). This paper proposes a novel algorithm referred to as {\em Neighborhood Conformal Prediction (\texttt{NCP})} to improve the efficiency of uncertainty quantification from CP for deep classifiers (i.e., reduce prediction set size). The key idea behind \texttt{NCP} is to use the learned representation of the neural network to identify $k$ nearest-neighbors calibration examples for a given testing input and assign them importance weights proportional to their distance to create adaptive prediction sets. We theoretically show that if the learned data representation of the neural network satisfies some mild conditions, \texttt{NCP} will produce smaller prediction sets than traditional CP algorithms. Our comprehensive experiments on CIFAR-10, CIFAR-100, and ImageNet datasets using diverse deep neural networks strongly demonstrate that \texttt{NCP} leads to significant reduction in prediction set size over prior CP methods.
\end{abstract}

\section{Introduction}

Recent advances in deep learning have allowed us to build models with high accuracy. However, to safely deploy these deep models in high-stake applications (e.g, medical diagnosis) for critical decision-making, we need theoretically-sound uncertainty quantification (UQ) to capture the deviation of the prediction from the ground-truth output. The UQ could take the form of a {\em prediction set} (a subset of candidate labels) for classification tasks. For example, in medical diagnosis, such prediction sets will allow a doctor to rule out harmful diagnoses such as stomach cancer even if the most likely diagnosis is a stomach ache. Conformal prediction (CP) \cite{vovk1999machine,vovk2005algorithmic,shafer2008tutorial} is a principled framework for UQ that provides formal guarantees for a user-specified {\em coverage}: ground-truth output is contained in the prediction set with a high probability $\alpha$ (e.g., 90\%) for classification. Additionally, UQ from CP is adaptive and will reflect the difficulty of testing inputs: size of the prediction set will be large for difficult inputs and small for easy inputs. 

There are two key steps in CP. First, in the prediction step, we use a trained model (e.g., deep neural network) to compute {\em conformity scores} which measure similarity between calibration examples and a testing input. Second, in the calibration step, we use the conformity scores on a set of calibration examples to find a threshold to construct prediction set which meets the coverage constraint (e.g., $\alpha$=90\%). The {\em efficiency} of UQ from CP \cite{sadinle2019least} is measured in terms of size of the prediction set (the smaller the better). There is an inherent trade-off between coverage and efficiency. For example, it is easy to achieve high coverage with low efficiency (i.e., large prediction set) by including all or most candidate labels in the prediction set. CP for classification is relatively under-studied. Recent work has proposed conformity scores based on ordered probabilities \cite{NEURIPS2020_244edd7e,angelopoulos2021uncertainty} for UQ based on CP for classification and do not come with theoretical guarantees about efficiency. The main research question of this paper is: {\em how can we improve CP to achieve provably higher efficiency by satisfying the marginal coverage constraint for (pre-trained) deep classifiers?} 

To answer this question, this paper proposes a novel algorithm referred to as {\em {\bf N}eighborhood {\bf C}onformal {\bf P}rediction (\texttt{NCP})} that is inspired by the framework of localized CP (LCP) \cite{guan2021localized}. The key idea behind LCP is to assign higher importance to calibration examples in the local neighborhood of a given testing input. This is in contrast to the standard CP, which assigns equal importance to all calibration examples. However, there is no theoretical analysis of LCP to characterize the necessary conditions for reduced prediction set size and no empirical evaluation on real-world classification tasks. The effectiveness of NCP critically depends on the localizer and weighting function. The proposed \texttt{NCP} algorithm specifies a concrete localizer and weighting function using the learned input representation from deep neural network classifiers. For a given testing input, \texttt{NCP} identifies $k$ nearest neighbors and assigns them importance weights proportional to their distance defined using the learned representation of deep classifier. 

We theoretically analyze the expected threshold of \texttt{NCP}, which is typically used to measure the efficiency of CP algorithms, i.e., a smaller expected threshold indicates smaller prediction sets.  Specifically, we prove that if the learned data representation of the neural network satisfies some mild conditions in terms of separation and concentration, \texttt{NCP} will produce smaller prediction sets than traditional CP algorithms. To the best of our knowledge, this is the first result to give affirmative answer to the open question: what are the necessary conditions for \texttt{NCP}-based algorithms to achieve improved efficiency over standard CP? Our theoretical analysis informs us with a principle to design better CP algorithms: automatically train better localizer functions to reduce the prediction set size for classification. 

We performed comprehensive experiments on CIFAR-10, CIFAR-100, and ImageNet datasets using a variety of deep neural network classifiers to evaluate the efficacy of \texttt{NCP} and state-of-the-art CP methods. Our results demonstrate that the localizer and weighting function of \texttt{NCP} is highly effective and results in significant reduction in prediction set size; and a better conformity scoring function further improves the efficacy of \texttt{NCP} algorithm.  Our ablation experiments clearly match our theoretical analysis of the \texttt{NCP} algorithm.

\vspace{1.0ex}

\noindent {\bf Contributions.} The key contribution of this paper is the development of Neighborhood CP (\texttt{NCP}) algorithm along with its theoretical and empirical analysis for improving the efficiency of uncertainty quantification of deep classifiers. 


\begin{itemize}
\setlength\itemsep{0em}
\item Development of the Neighborhood CP algorithm by specifying an effective localizer function using the learned representation from deep classifiers. \texttt{NCP} is complementary to CP methods, i.e., any advancements in CP will automatically benefit \texttt{NCP} for uncertainty quantification.

\item Novel theoretical analysis of the Neighborhood CP algorithm to characterize when and why it will improve the efficiency of UQ over standard CP methods.
    
\item Experimental evaluation of \texttt{NCP} on classification benchmarks using diverse deep models to demonstrate its efficacy over prior CP methods. Our code is publicly available on the GitHub repository: \href{https://github.com/1995subhankar1995/NCP}{https://github.com/1995subhankar1995/NCP}.
\end{itemize}


\section{Background and Problem Setup}


We consider the problem of uncertainty quantification (UQ) of pre-trained deep models for classification tasks. Suppose $x$ is an input from the space $\mathcal{X}$ and $y^* \in \mathcal{Y}$ is the corresponding ground-truth output. 
Let $\calZ$ = $\calX \times \calY$ be the joint space of input-output pairs and the underlying distribution on $\calZ$ be $\calD_{\calZ}$.
For classification tasks, $\mathcal{Y}$ is a set of $C$ discrete class-labels $\{1, 2, \cdots, C\}$. As per the standard notation in conformal prediction, $X$ is a random variable, and $x$ is a data sample.

\vspace{1.0ex}

\noindent {\bf Uncertainty quantification.} Let $\calD_\tr$ and $\calD_\cal$ correspond to sets of training and calibration examples drawn from a target distribution $\calD_{\calZ}$.
We assume the availability of a pre-trained deep model $F_{\theta}: \mathcal{X} \mapsto \mathcal{Y}$, where $\theta$ stands for the parameters of the deep model. For a given testing input $x$, we want to compute UQ of the deep model $F_{\theta}$ in the form of a prediction set $\mathcal{C}(x)$, a subset of candidate class-labels $\{1, 2, \cdots, C\}$. 

\vspace{1.0ex}

\noindent {\bf Coverage and efficiency.} The performance of UQ is measured using two metrics. First, the (marginal) {\em coverage} is defined as the probability that the ground-truth output $y^*$ is contained in $\mathcal{C}(x)$ for a testing example $(x, y^*)$ from the same data distribution $\calD_{\calZ}$, i.e., $\mathbb{P}(y^* \in \mathcal{C}(x))$. The empirical coverage \texttt{Cov} is measured over a given set of testing examples $\calD_\test$. Second, {\em efficiency}, denoted by \texttt{Eff}, measures the cardinality of the prediction set $\mathcal{C}(x)$ for classification. Smaller prediction set means higher efficiency. It is easy to achieve the desired coverage (say 90\%) by always outputting $\mathcal{C}(x)$=$\mathcal{Y}$ at the expense of poor efficiency. 

\vspace{1.0ex}

\noindent {\bf Conformal prediction (CP).} CP is a framework that allows us to compute UQ for any given predictor through a conformalization step. The key element of CP is a measure function $V$ to compute the {\em conformity} (or {\em non-conformity}) score, measures similarity between labeled examples, which is used to compare a given testing input to the calibration set $\calD_\cal$. 
Since non-conformity score can be intuitively converted to a conformity measure \cite{vovk2005algorithmic}, we use non-conformity measure for ease of technical exposition.  A typical method based on split conformal prediction (see Algorithm 1) has a threshold parameter $\tau \rightarrow t$ to compute UQ in the form of prediction set for a given testing input $x$ and deep model $F_{\theta}$. A small set of calibration examples $\calD_\cal$ are used to select the threshold $t$ for achieving the given coverage $1-\alpha$ (say 90\%) empirically on  $\calD_\cal$.  For example, in classification tasks, we select the $t$ as $(1-\alpha)$-quantile of $V(x,y^*)$ on the calibration set  $\calD_\cal$ and the prediction set for a new testing input $x$ is given by $\mathcal{C}(x)$=$\{y: V(x, y) \le t\}$. CP provides formal guarantees that $\mathcal{C}(x)$ has coverage $1-\alpha$ on a future testing input from the same distribution $\calD_{\calZ}$.

\begin{algorithm}[!h]

    \caption{Split Conformal Prediction (CP)} 
    \label{alg:gen_CP}
    \begin{algorithmic}[1]
    
    \STATE \textbf{Input}: Significance level $\alpha \in (0, 1)$;
    Randomly split data into training set $\calD_\tr$ and calibration set $\calD_\cal = \{ Z_1,\cdots, Z_n \}$.

    \STATE If predictor $F_\theta$ is not given, train a prediction model $F_\theta$ on the training set $\calD_\tr$
    
    \STATE Compute non-conformity score $V_i$ for each example $Z_i \in \calD_\cal$

    \STATE Compute $\hat{Q}^\CP(\alpha, V_{1:n})$ as the $\lceil (1-\alpha)(1 + |\mathcal{D}_{cal}|) \rceil$th smallest value in $\{V_i\}_{i \in \calD_\cal }$ as in (\ref{eq:CP}).

    \STATE $\hat {\mathcal{C}}(x_{n+1}) = \{ y : V(x_{n+1}, y) \leq \hat Q^\CP(\alpha, V_{1:n}) \} $ is the prediction set for a testing input $x_{n + 1}$
    \end{algorithmic}
   
\end{algorithm}
 
For classification, recent work has proposed conformity scores based on ordered probabilities \cite{NEURIPS2020_244edd7e,angelopoulos2021uncertainty}. The conformity score of adaptive prediction sets (APS) \cite{NEURIPS2020_244edd7e} is defined as follows. For a given input $x$, we get the sorted probabilities for all classes using the deep model $F_\theta$, $\pi(x,y^1) \geq \cdots \pi(x,y^C)$, and compute the score: 
    \begin{equation*}
        V^{\text{APS}}(x, k) =  \pi(x,y^1) + \cdots + \pi(x,y^{k-1}) + U \cdot \pi(x,y^{k})
    \end{equation*}
where $U \in [0,1]$ is a random variable to break ties. Suppose $L$ is the index of the ground-truth class label $y^*$ in the ordered list of probabilities $\pi(x,y)$. The conformity scoring function of regularized APS (RAPS) \cite{angelopoulos2021uncertainty} is: 
    \begin{equation*}
        V^{\text{RAPS}}(x, k) =  V_{APS}(x, k) + \lambda_R \cdot |L - k_{reg}|
    \end{equation*}
where $\lambda_R$ is the regularization parameter and $k_{reg}$ is another parameter which is set based on the distribution of $L$ values on validation data. 
Since \texttt{NCP} is a wrapper algorithm, we will employ both APS and RAPS to demonstrate the effectiveness of \texttt{NCP} in improving efficiency.

\vspace{1.0ex}

\noindent {\bf Problem definition.} The high-level goal of this paper is to study provable methods to improve the standard CP framework to achieve high-efficiency (small prediction set) by meeting the coverage constraint $1-\alpha$. Specifically, we propose and study the neighborhood CP algorithm that assigns higher importance to calibration examples in the local neighborhood of a given testing example. We theoretically and empirically analyze when/why neighborhood CP algorithm results in improved efficiency.  We provide a summary of the mathematical
notations used in this paper in Table \ref{tab:math_notation}.
\begin{table}[!h]
\centering
\caption{\textbf{Key mathematical notations used in this paper.}}
\label{tab:math_notation}
\begin{tabular}{|l|l|}
    \hline
     {\bf Notation} & {\bf Definition} \\
     \hline
     $x \in \mathcal{X} $ & input example \\
     \hline
     $y^* \in \mathcal{Y}$ & ground-truth output \\
     \hline
     $\mathcal{Z} = \mathcal{X} \times \mathcal{Y}$& joint space of input-output pairs \\
     \hline
     $F_{\theta}$ & neural network with parameters $\theta$\\
     \hline
     $\mathcal{C}(x)$ & prediction set for input $x$\\
     \hline
     $V((x, y^*) \in \mathcal{Z})$ & non-conformity scoring function \\
     \hline
     $\alpha$ & mis-coverage rate \\
     \hline
     $\hat{Q}(\alpha)$ & $(1 - \alpha)$ quantile of $V((x, y^*) \in \mathcal{D}_{cal}$\\
     \hline
     $\lambda_{R}$ & regularization hyper-parameter\\
     \hline
     $\lambda_{L}$ & localization hyper-parameter\\
     \hline
     $\Phi(x)$ & learned representation of $x$ from $F_{\theta}$ \\
     \hline
\end{tabular}
\end{table}

\section{Neighborhood Conformal Prediction}


In this section, we provide the details of the Neighborhood CP (\texttt{NCP}) algorithm to improve the efficiency of CP. We start with some basic notation and terminology. 

Let $V: \calZ \rightarrow \R$ denote the non-conformity measure function. For each data sample $Z_i$ in the calibration set $\calD_\cal$, we denote the corresponding non-conformity score as $V_i$ = $V(Z_i)$. For a pre-specified level of coverage $1-\alpha$, we can determine the quantile on $\calD_\cal$ as shown below.
\begin{align}\label{eq:CP}
\hat Q^\CP(\alpha, V_{1:n}) = \min\Big\{ t : \sum_{i=1}^n \frac{1}{n} \cdot \indicator[ V_i \leq t ] \geq 1 - \alpha  \Big\}
\end{align}
where $t$ is the threshold parameter, \indicator\; is the indicator function, and $n$ is the number of calibration examples, i.e., $n$=$|\calD_\cal|$. We can measure the {\it efficiency} of conformal prediction using this quantile. Given a fixed $\alpha$, a smaller quantile means a smaller threshold $t$ to achieve the same significance level $\alpha$, leading to more efficient uncertainty quantification (i.e., smaller prediction set). We provide a generic split conformal prediction algorithm in Algorithm \ref{alg:gen_CP} for completeness.

In many real-world applications, the conditional distribution of output $y$ given input $x$ can vary due to the heterogeneity in inputs. Thus, it is desirable to exploit this inherent heterogeneity during uncertainty quantification. The key idea behind the localized conformal prediction framework \cite{guan2021localized} is to assign importance weights to conformal scores on calibration examples: higher importance is given to calibration samples in the local neighborhood of a given testing example $X_{n+1}$. In contrast, standard CP assigns uniform weights to calibration examples. We summarize the \texttt{NCP} procedure in Algorithm \ref{alg:LCP}. Importantly, \texttt{NCP} is a wrapper approach in the sense that it can be used with any conformity scoring function. For example, if we use a better conformity score (say RAPS), we will get higher improvements in efficiency of UQ through \texttt{NCP} as demonstrated by our experiments.

Suppose $p_{i,j}$ denotes a weighting function that accounts for the similarity between any $i$-th and $j$-th calibration samples from $\calD_\cal$. To determine the \texttt{NCP} quantile given a mis-coverage level $\alpha$, it suffices to find $\hat \alpha^\NCP(\alpha)$ as below:

\begin{align}
    \begin{split}
        \tilde{\alpha}^{NCP}(\alpha) = \max\{ \tilde{\alpha}:\sum_{i=1}^{n}\frac{1}{n}.\indicator[V_i \leq \tilde{Q}^{NCP}(\tilde{\alpha};V_{1:n};p_{i,1:n})] \} \\ \geq 1 - \alpha
    \end{split}
    \label{eq:LCP_tilde_alpha}
\end{align}

where the \texttt{NCP} quantile is determined as follows:
\begin{align*}
{\textstyle\hat Q^\NCP(\tilde \alpha; V_{1:n}; p_{i,1:n})
=
 \min \{ t :} \sum_{j=1}^n p_{i,j} \indicator[ V_j \leq t ] \geq 1 - \tilde \alpha \}
\end{align*}

Prior work on localized CP \cite{guan2021localized} mainly focused on theoretical analysis of finite-sample coverage guarantees, synthetic experiments using raw inputs $x$ for the regression task, and acknowledged that design of localizer and weighting function for practical purposes is out of their scope. Since the efficiency of UQ from LCP critically depends on the localizer and weighting function, our \texttt{NCP} algorithm specifies an effective localizer that leverages the learned input representation $\Phi(x)$ from the pre-trained deep model $F_\theta$ (e.g., the output of the layer before the softmax layer in CNNs for image classification) for our theoretical and empirical analysis. As we show in our theoretical analysis, input representation $\Phi(x)$ exhibiting good separation between inputs from different classes will result in small prediction sets.

\begin{algorithm}[!h]
    \caption{Neighborhood Conformal Prediction (NCP) }
    \label{alg:LCP}
    \begin{algorithmic}[1]
    
    \STATE \textbf{Input}: Significance level $\alpha \in (0, 1)$. 
    Randomly split data into training set $\calD_\tr$ and calibration set $\calD_\cal = \{ Z_1,\cdots, Z_n \}$.

    \STATE If predictor $F_\theta$ is not given, train a prediction model $F_\theta$ on the training set $\calD_\tr$. 
    
    \STATE Compute non-conformity score $V_i$ for $Z \in \calD_\cal$

    \STATE Compute importance weights $p_{i, j}$ for $X_i, X_j \in \calD_\cal$ needed  to identify {\em neighborhood} according to (\ref{eq:ball_based_localizer} or \ref{localizer_sum_1})

    \STATE Find $\alpha^\NCP$ in (\ref{eq:LCP_tilde_alpha}) on $\calD_\cal$ and set 
     $\hat {\mathcal{C}}(X_{n+1}) = \{ y : V(X_{n+1}, y) \leq \hat Q(\alpha^\NCP, V_{1:n+1}, p_{n+1,1:n+1}) \} $ as the prediction set for the new data sample $X_{n + 1}$ 
    \end{algorithmic}
\end{algorithm}

\newpage

\noindent {\bf Localizer and weighting function.}  We propose a ball-based localizer function for our \texttt{NCP} algorithm as shown below.
\begin{align}\label{eq:ball_based_localizer}
p_{i,j} = \frac{ \indicator[ \Phi(X_j) \in \calB( \Phi({X_i}) ) ] }{ \sum_{k \in \calD_\cal} \indicator[ \Phi(X_k) \in \calB( \Phi(X_i) ) ] }, 
\end{align}
where $\Phi(X)$ is the learned representations of $X$ by the deep model $F_\theta$, \indicator\; is the indicator function, and $\calB(x) \triangleq \{ x' \in \calX: \| x - x'\| \leq B \}$ denotes a Euclidean ball with radius $B$ centered at $x$.
In the next section, we theoretically analyze the improved efficiency of \texttt{NCP} (Algorithm \ref{alg:LCP}) over traditional CP algorithms under some mild conditions over the learned input representation by (pre-trained) deep classifier.

In some scenarios, a Euclidean ball of fixed radius to define the neighborhood of a testing example may fail as we may not find any calibration examples in the neighborhood. Therefore, we propose the following weighting function $p_{i,j} \leftarrow {p_{i, j}^{exp}}$ over $k$ nearest neighbors of a sample. 
\begin{equation}
        IW(x_{i}, x_{j}) = \exp {\left(-\frac{dist(\Phi(x_{i}), \Phi(x_j))}{\lambda_{L}}\right)}
        \label{localizer}
\end{equation}

{
\begin{equation}
        p_{i, j}^{exp} = \frac{IW(x_{i}, x_{j})}{\sum_{k\in \calD_{cal}(KNN)}IW(x_{i}, x_{k})}
        \label{localizer_sum_1}
\end{equation}
}
where {$dist(\Phi(x_{i}),\Phi(x_{j}))$} is a distance measure to capture the dissimilarity, which is the Euclidean distance in our case; $\lambda_L$ is a tunable hyper-parameter (smaller value results in more localization); and non-zero importance weights are assigned to {\em only} the $k$ nearest neighbors, i.e., $IW(x_{i}, x_{j})$=0 if $x_j$ is {\em not} one of the $k$ nearest-neighbors of $x_i$. Our theoretical analysis is applicable to any localizer and weighting function defined in terms of $\Phi(.)$, but we employ the localizer based on Equation (\ref{localizer_sum_1}) for our experiments. Our ablation experiments (see Table 1 and Table 2) comparing \texttt{NCP} with a variant that includes all calibration examples in the neighborhood justify the appropriateness of selecting $k$ nearest neighbors.

\section{Theoretical Analysis}
\label{section:theory}

In this section, we theoretically analyze the efficiency of \texttt{NCP} when compared with standard CP. Recall that efficiency of a conformal prediction approach is related to the size of prediction set (small value means higher efficiency) for a pre-defined significance level $\alpha$. Before providing the analysis, we first give our criterion for efficiency comparison. Since the quantile increases as $\alpha$ increases, we employ the following approach for comparing \texttt{NCP} and CP. We regard the quantile value as a function of the pre-defined significance level $\alpha$.
Specifically, for CP, there is a constant quantile for all data samples, i.e., $Q^\CP(\alpha)$, while for \texttt{NCP}, the quantile is adaptive and depends on the data sample. Therefore, we use expected quantile value as the measure of efficiency of \texttt{NCP}, denoted by $\bar Q^\NCP(\alpha)$.
To directly compare the efficiency of both algorithms, we fix the same $\alpha$ and compare the two quantile functions. We first determine the quantile using CP for achieving the target significance level $\alpha$. Next, we plug this CP quantile into \texttt{NCP} to show that \texttt{NCP} can always achieve higher coverage under some conditions on the distribution of learned input representation by deep classifier and the alignment with non-conformity measures on such distribution.
The main reason for requiring synchronizing the quantile and to compare the coverage is that \texttt{NCP} may have adaptive quantile depending on the specific data sample, while CP has a fixed quantile for all samples. Hence, it is difficult to compare their quantiles given the same $\alpha$.

We assume that the data samples from feature space $\calX$ can be classified into $C$ classes.
We assume that data samples are drawn from distribution $\calD_\calX$ and the ground-truth labels are generated by a target function $F^*\footnote{We used $F^*(x) = y^* = y$.}: \calX \rightarrow [C]$.
For each $c \in [C]$, let $\calX_c = \{ x \in \calX : F^*(x) = c \}$ denote the samples with ground truth label $c$ and $P_\class^{min} = \min_{c \in [C]} \P_X\{ X \in \calX_c \}$.
Our analysis relies on the concept of neighborhood of data sample(s). 
$\calN_B(x)$ be the neighborhood of $x$, i.e., {$\calN_\calB(x) \triangleq \{ x': \calB(\Phi(x)) \cap \calB(\Phi(x')) \neq \emptyset \}$}.

We define a robust set of $F^*$ by $\calR_\calB^* \triangleq \{ x \in \calX: F^*(x) = F^*(x') , \forall x' \in \calN_\calB(x)\}$, in which all elements share the same label with its neighbors.
Below we introduce several important definitions.

\begin{definition}\label{definition:concentration_neighbor_of_robust_set}
($\sigma$-concentration of quantiles on $\calR_\calB^* \cap \calX_c)$
Let $\sigma \geq 1$ be some constant.
If for any quantile $t$ of the non-conformity scores and any $X \in \calR_\calB^* \cap \calX_c$ where $c \in [C]$, we have
\begin{align*}
\P_{X'}& \{ V(X', F^*(X')) \leq t, X' \in \calN_\calB(X) | X \in \calR_\calB^* \cap \calX_c \}
\\ &\geq 
\sigma \cdot \P_{X'} \{ V(X', F^*(X')) \leq t | X' \in \calX_c \},
\end{align*}
then the quantiles are $\sigma$-concentrated on $\calR_\calB^* \cap \calX_c$.
\end{definition}

Definition \ref{definition:concentration_neighbor_of_robust_set} states that samples satisfying any quantile of non-conformity scores $V$ are distributed {\it significantly densely} on $\calR_\calB^* \cap \calX_c$.
Considering only those samples with less than $t$ non-conformity scores and assuming they are distributed uniformly in the entire $\calX_c$, the associated condition number
$\sigma = 1$.
If the distribution of such samples can be more concentrated on $\calR_\calB^* \cap \calX_c$ rather than the non-robust set $\calX_c \backslash \calR_\calB^*$, then we have a larger condition number $\sigma > 1$.

\begin{definition}\label{definition:separation}
($\mu_\calB$-separation)
If $\max_{c \in [C]} \P_X \{X \notin \calR_\calB^*, F^*(X) = c \} \leq \mu_\calB$, then the underlying distribution $\calD_\calX$ is $\mu_\calB$-separable.
\end{definition}

Definition \ref{definition:separation} states that the region of robust set is lower bounded for each class, i.e., $\P_X \{ X \in \calR_\calB^* , F^*(X) = c \} \geq 1 - \mu_\calB$.
A smaller condition number $\mu_\calB$ means that more percentage of data samples are distributed in the robust set.
The above definitions are built on population distribution rather than training set, as the expansion condition in \cite{wei2020theoretical}.

\begin{assumption}\label{assumption:LCP}
Suppose $\alpha$ is the pre-defined significance level and:
(i) any quantile $t$ satisfies $\sigma$-concentration on $\calR_\calB^* \cap \calX_c$ for $c \in [C]$;
(ii) $\calD_\calX$ is $\mu_\calB$ separable;

(iii) the condition numbers in (i) and (ii) satisfy 
$(1-\mu_\calB) \sigma \geq ( 1 - \alpha ) / ( 1 - \alpha ( 2 - P_\class^{min} ) )$;
(iv) $\alpha\leq 1/2$.
\end{assumption}

We make assumptions for any quantile $t$ of non-conformity scores at the population level.
Given the above definition of the neighborhood $\calN_\calB$, $\mu_\calB$ implicitly describes the difficulty of the classification problem.

The assumption on the concentration of quantiles requires that the non-conformity scores $V$ are consistent with the underlying distribution of the input representation $\calD_\calX$.
Using the concepts of robust set $\calR_\calB^*$ and its neighborhood region $\calN_\calB(\calR_\calB^*)$ (or its class-wise component $\calN_\calB(\calR_\calB^* \cap \calX_c)$ as in Definition \ref{definition:concentration_neighbor_of_robust_set}), 
we can partition the entire sample space $\calX$ into two disjoint parts: {\it core part} (robust set), and the {\it remaining part} (outside the robust set).
Good input representation $\calX$ will achieve good separation of data clusters for different classes, i.e., the above condition number $(1-\mu_\calB)\sigma$ can be larger than $1$.

\begin{theorem}\label{theorem:improved_LCP_over_CP}
Suppose Assumption \ref{assumption:LCP} holds.

Let
$Q^\CP(\alpha) \triangleq \min\{ t : \P_X \{ V(X, F^*(X)) \leq t \} \geq 1 - \alpha \}$,
$\alpha^\NCP(\alpha) \triangleq \max\{ \tilde\alpha : \P_X \{ X \leq Q^\NCP(\tilde \alpha; X)\} \geq 1 - \alpha \}$,
where $Q^\NCP(\tilde \alpha; X) \triangleq \min\{ t : \P_{X'}\{ V(X'; F^*(X)) \leq t, X' \in \calN_\calB(X) \} \geq 1 - \tilde \alpha \}$ in population.

Then to achieve the same $\alpha$, \texttt{NCP} gives smaller expected quantile and can be more efficient than CP:

$
\E_{X} [ Q^\NCP(\alpha; X) ] \leq Q^\CP(\alpha).
$
\end{theorem}

\begin{figure*}[t]
    \centering
    \includegraphics[width=\linewidth]{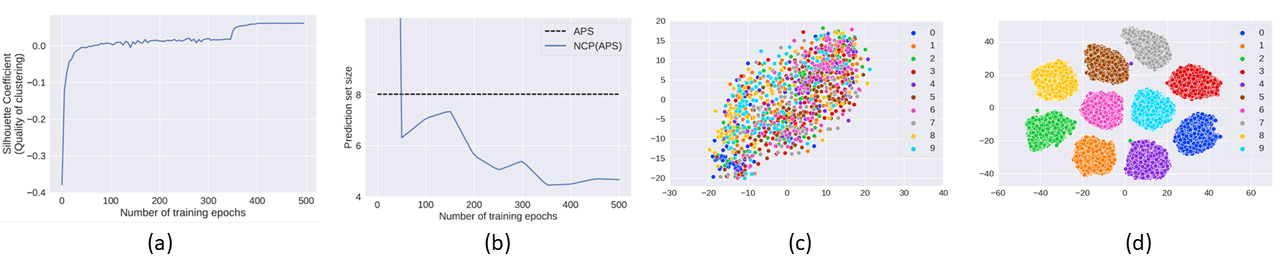}
    \caption{Ablation results to justify \texttt{NCP} algorithm and its theoretical analysis. ({\bf Left}) As a function of the number of training epochs for ResNet50  trained on CIFAR100, we show (a) the clustering property of learned input representation and (b) prediction set size from \texttt{NCP}. The conformity score is kept constant for both CP and \texttt{NCP}. The prediction set size from \texttt{NCP} reduces over CP as the clustering property of the learned representation improves. ({\bf Right}) the t-SNE visualization of CIFAR10 data categorized by their ground-truth class labels using (c) raw image pixels and (d) learned input representation from ResNet50 model after 500 training epochs. The learned representation exhibits significantly a better clustering property.}
    \label{fig:t_SNE_all}
\end{figure*}

{\bf Proof.} (Sketch) We present the complete proof in Appendix due to space 
constraints. Our theoretical analysis proceeds with two key milestones. 
In the first step, we establish a connection from \texttt{NCP} to an artificial CP algorithm referred to as {\it class-wise \texttt{NCP}} that determines the threshold for each class independently. The gap between \texttt{NCP} and class-wise \texttt{NCP} can be filled by using the assumed $\sigma$-concentration and $\mu_\calB$-separation, since both conditions characterize the properties of input representation along different dimensions.
Our analysis shows a clear improvement of efficiency by \texttt{NCP}, i.e., smaller expected quantile over the data distribution when compared to class-wise \texttt{NCP}.
In the second step, we build a link between class-wise \texttt{NCP} with traditional CP.
Even though they can both achieve marginal coverage, class-wise \texttt{NCP} finds the quantiles for individual classes {\em adaptively} (i.e., quantiles may vary for different classes). However, traditional CP can only determine a {\em global} quantile across the entire data distribution. Finally, we combine these two milestones to connect \texttt{NCP} to traditional CP to show improved efficiency. 

\vspace{0.25ex}

\begin{remark}
\normalfont The above result treats the quantiles determined by \texttt{NCP} and CP as a function of $\alpha$.
When achieving the same $\alpha$, the expected quantile derived by \texttt{NCP} is smaller than the one determined by CP.
The key idea behind this result is to make use of the alignment between the quantiles and distribution, 
which critically depends on the learned input representation of $\calX$ by deep classifiers. Learned input representations with a good clustering property will reduce the prediction set size.
The sample distribution is characterized by the robust set $\calR_\calB^*$ and the remaining region.
Additionally, if the non-conformity scores are sufficiently consistent with data distribution in any quantiles, then a well-designed localizer can reduce the required quantile to achieve $\alpha$, the target significance level.
\end{remark}


\section{Experiments and Results}

We present experimental evaluation of \texttt{NCP} over prior CP methods for classification and other baselines to demonstrate the effectiveness of \texttt{NCP} in reducing prediction set size.

\subsection{Experimental Setup}

\vspace{0.75ex}

\noindent {\bf Classification datasets.} We consider CIFAR10 \cite{krizhevsky2009learning}, CIFAR100 \cite{krizhevsky2009learning}, and ImageNet \cite{deng2009imagenet} datasets using the standard training and testing set split. We employ the same methodology as \cite{angelopoulos2021uncertainty} to create calibration data and validation data for tuning hyper-parameters.

\vspace{1.0ex}
\noindent {\bf Regression datasets.} We consider eleven UCI benchmark datasets from five different sources: Blog feedback (BlogFeedback) \cite{akbilgic2014novel}, 
Gas turbine NOx emission (Gas turbine) \cite{kaya2019predicting}, five Facebook comment volume (Facebook\_$i$) \cite{Sing1601:Facebook}, three medical datasets(MEPS\_$i$)\cite{MEPS_data}, and Concrete\cite{yeh1998modeling}. We employ a similar dataset split as the classification task.
\vspace{1.0ex}

\noindent {\bf Deep models} \textit{Classification:} We train three commonly used ResNet architectures, namely,  ResNet18, ResNet50, and ResNet101 \cite{he2016deep} on CIFAR10 and CIFAR100. For ImageNet, we employ nine pre-trained deep models similar to \cite{angelopoulos2021uncertainty}, namely, ResNeXt101, ResNet152, ResNet101, ResNet50, ResNet18, DenseNet161, VGG16, Inception, and ShuffleNet from the TorchVision repository \cite{paszke2019pytorch}. We calibrate 
classifiers via Platt scaling \cite{guo2017calibration} on the calibration data before applying CP methods.

\textit{Regression: }We consider a multi-layer perceptron with three hidden layers with respectively 15, 20, and 30 hidden nodes. We train this neural network for different datasets by optimizing the Mean Squared Error loss. We employ Adam optimizer\cite{kingma2014adam} and train for 500 epochs with a learning rate $10^{-3}$.
\begin{table}[t]
\centering
\caption{The average percentage of calibration examples used by \texttt{NCP} as nearest neighbors  to produce prediction sets.}

\label{tab:K_fraction_ncp}
\begin{tabular}{|c|c|c|c|}

\hline
Model     & ImageNet & CIFAR100 & CIFAR10 \\ \hline
ResNet18  & 24.8     & 12.3     & 09.0     \\ \hline
ResNet50  & 12.8     & 08.3      & 15.0    \\ \hline
ResNet101 & 33.2     & 12.0     & 13.0    \\ \hline
\end{tabular}
\end{table}
\vspace{1.0ex}

\noindent {\bf Methods and baselines.} \textit{Classification:} We consider the conformity score of \texttt{APS} \cite{romano2020classification} and \texttt{RAPS} \cite{angelopoulos2021uncertainty} as strong CP baselines. Since \texttt{NCP} is a wrapper algorithm that can work with any conformity score, our experiments will demonstrate the efficacy of \texttt{NCP} to improve over both APS and RAPS, namely, \texttt{NCP(APS)} and \texttt{NCP(RAPS)}. We employ the publicly available implementation\footnote{\url{https://github.com/aangelopoulos/conformal_classification/}} of APS and RAPS, and build on it to implement our \texttt{NCP} algorithm. We also compare with a \texttt{naive} baseline \cite{angelopoulos2021uncertainty} that produces a prediction set by including classes from highest to lowest probability until their cumulative sum just exceeds the threshold $1 - \alpha$. 

\textit{Regression:} We consider the Absolute Residual\cite{vovk2005algorithmic} as a non-conformity score:
\begin{equation}
    V(x,y^*)=|y^* - F_{\theta}(x)|
\end{equation} 
to create a baseline \texttt{CP} for the regression task using Algorithm \ref{alg:gen_CP} and we compare it to our proposed \texttt{NCP} method.

\vspace{1.0ex}

\noindent {\bf Evaluation methodology.} We select the hyper-parameters for RAPS and \texttt{NCP} from the following values noting that Bayesian optimization \cite{BO1,BO2} can be used for improved results: $\lambda_{L} = \{10, 50, 100, 500, 1000, 5000\}$ and $\lambda_{R} = \{0.001, 0.005, 0.01, 0.05, 0.15, 0.2, 0.3, 0.4, 0.5, 1.0\}$. 
We create calibration and validation data following the methodology in \cite{angelopoulos2021uncertainty} and use validation data for tuning the hyper-parameters. We present all our experimental results for desired coverage as 90\% (unless specified otherwise). 
We compute two metrics computed on the testing set: {\em Coverage} (fraction of testing examples for which prediction set contains the ground-truth output) and {\em Efficiency} (average length of cardinality of prediction set for classification and average length of the prediction interval for regression) where small values mean high efficiency. We report the average metrics over five different runs for ImageNet and ten different runs for all other datasets.

\begin{table*}[t]
\centering
\caption{Ablation results comparing \texttt{NCP} and \texttt{NCP} variant using all calibration examples as neighborhood \texttt{(NCP-All)}. \texttt{NCP} reduces the mean prediction set size by using a smaller neighborhood. Both \texttt{NCP} and \texttt{NCP-All} achieve $1 - \alpha$ coverage.}
\label{tab:LCP_NCP}
\begin{tabular}{|c|cc|cc|cc|}
\hline
\multicolumn{1}{|c|}{Dataset}   & \multicolumn{2}{c|}{ImageNet}             & \multicolumn{2}{c|}{CIFAR100}             & \multicolumn{2}{c|}{CIFAR10}              \\ \hline
                      $1 - \alpha$          & \multicolumn{2}{c|}{$0.90$} & \multicolumn{2}{c|}{$0.90$} & \multicolumn{2}{c|}{$ 0.96$} \\ \hline
\multicolumn{1}{|c|}{}
                                & \multicolumn{1}{c|}{NCP-All}       & NCP      & \multicolumn{1}{c|}{NCP-All}       & NCP      & \multicolumn{1}{c|}{NCP-All}       & NCP      \\ \hline
\multicolumn{1}{|c|}{ResNet18}  & \multicolumn{1}{c|}{11.82}     & \textbf{11.08}    & \multicolumn{1}{c|}{4.87}      & \textbf{3.61}     & \multicolumn{1}{c|}{1.22}      & \textbf{1.07}     \\ \hline
\multicolumn{1}{|c|}{ResNet50}  & \multicolumn{1}{c|}{09.24}      & \textbf{06.44}     & \multicolumn{1}{c|}{7.45}      & \textbf{4.26}     & \multicolumn{1}{c|}{1.16}      & \textbf{1.05}     \\ \hline
\multicolumn{1}{|c|}{ResNet101} & \multicolumn{1}{c|}{07.90}       & \textbf{05.60}      & \multicolumn{1}{c|}{4.17}      & \textbf{2.80}     & \multicolumn{1}{c|}{1.18}      & \textbf{1.04}     \\ \hline
\end{tabular}
\end{table*}

\begin{table*}[!h]
\centering
\caption{\textbf{Mean prediction set size} on ImageNet, CIFAR100, and CIFAR10. We present the mean and standard deviation over five different runs for ImageNet and ten different runs for CIFAR100 and CIFAR10 respectively.} 
\label{tab:CIFAR100_size_mod}
\begin{tabular}{|c|c|c|c|c|c|}
\hline
Model & Naive & APS & RAPS & NCP(APS) & NCP(RAPS)\\ \hline
\multicolumn{6}{|c|}{ImageNet($1 - \alpha = 0.900$)} \\ \hline
ResNet152 &  10.55(0.548) & 11.34(0.843) & 2.53(0.054) & 5.24(0.312) & \textbf{2.12(0.007)} \\ \hline
ResNet101 &  10.85(0.588) & 11.36(0.615) & 2.63(0.08) & 5.60(0.278) & \textbf{2.25(0.098)} \\ \hline
ResNet50  & 12.00(0.530) & 13.24(0.892) & 2.94(0.089) & 6.44(0.381) & \textbf{2.54(0.091)} \\ \hline
ResNet18  & 16.13(0.642) & 17.05(1.100) & 5.00(0.220) & 11.08(0.598) & \textbf{4.71(0.251)} \\ \hline
ResNeXt101  & 18.57(1.05) & 20.57(1.10) & 2.36(0.069) & 8.24(0.388) & \textbf{2.01(0.060)} \\ \hline
DenseNet161  & 11.70(0.730) & 12.86(1.21) & 2.67(0.074) & 5.89(0.447) & \textbf{2.29(0.103)} \\ \hline
VGG16  & 13.91(0.867) & 14.10(0.486) & 3.97(0.098) & 8.56(0.382) & \textbf{3.62(0.106)} \\ \hline
Inception & 77.51(3.78) & 90.28(4.16) & 5.96(0.373) & 66.29(2.83) & \textbf{5.67(0.178)} \\ \hline
ShuffleNet  & 30.33(1.64) & 34.31(2.10) & 5.53(0.070) & 22.36(1.37) & \textbf{5.49(0.131)} \\ \hline
\multicolumn{6}{|c|}{CIFAR100 ($1 - \alpha = 0.900$)} \\ \hline
ResNet18 &   5.07(0.327) & 5.57(0.224)  &3.17(0.076) & 3.61(0.164) & \textbf{2.77(0.100)}    \\ \hline
ResNet50   & 8.21(0.746) & 8.02(0.405) & 3.37(0.189) & 4.26(0.233) & \textbf{2.96(0.213)}    \\ \hline
ResNet101   & 4.44(0.348) & 4.64(0.184) & 2.60(0.066) & 2.80(0.188) & \textbf{2.19(0.073)} \\ \hline
\multicolumn{6}{|c|}{CIFAR10($1 - \alpha = 0.960$)} \\ \hline
ResNet18 &  1.20(0.016) & 1.25(0.021) & 1.24(0.014) & 1.07(0.013) & \textbf{1.06(0.012)}\\ \hline
ResNet50   & 1.15(0.013) & 1.19(0.021) & 1.17(0.011) & 1.05(0.010) & \textbf{1.04(0.005)}\\ \hline
ResNet101  & 1.17(0.017) & 1.20(0.018) & 1.19(0.017) & 1.04(0.008) & \textbf{1.03(0.007)}\\ \hline

\end{tabular}
\end{table*}

\subsection{Results and Discussion}

\vspace{0.75ex}

\noindent {\bf Empirical support for the theory of \texttt{NCP}.} Figure \ref{fig:t_SNE_all} (a) and (b) shows that as the clustering property of the learned input representation improves, prediction set size from \texttt{NCP(APS)} reduces and improves over APS. This empirical result clearly demonstrates our theoretical result comparing \texttt{NCP} and CP. Figure \ref{fig:t_SNE_all} (c) and (d) show the t-SNE visualization of raw image pixels and learned input representation from ResNet50 on CIFAR10 data. This result demonstrates that learned representation exhibits significantly better clustering property over raw data and justifies its use for the \texttt{NCP}'s localizer.

\vspace{0.75ex}

\noindent {\bf Ablation results for \texttt{NCP}.} Table \ref{tab:K_fraction_ncp} shows the fraction of calibration examples used by \texttt{NCP(APS)} on an average to compute prediction sets for three ResNet models noting that we find similar patterns for other deep models. These results demonstrate that a small fraction of calibration examples are used by \texttt{NCP(APS)} as neighborhood. One interesting ablation question is how does the \texttt{NCP} variant using all calibration examples compare to \texttt{NCP}? To answer this question, Table \ref{tab:LCP_NCP} shows the mean prediction set size for \texttt{NCP(APS)-All} and \texttt{NCP(APS)}. The results clearly demonstrate that \texttt{NCP(APS)} produces significantly smaller predicted set sizes justifying the selection of $k$ nearest neighbors in \texttt{NCP}. Finally, we provide in the Appendix results comparing the runtime of NCP and the corresponding CP baseline. Results demonstrate the negligible runtime overhead of employing NCP as a wrapper with any conformity scoring function.

\begin{table*}[!h]
\centering
\caption{\textbf{Marginal coverage} on ImageNet, CIFAR100, and CIFAR10. We present the mean and standard deviation over five different runs for ImageNet and ten different runs for CIFAR100 and CIFAR10 respectively.}
\label{tab:CIFAR100_cvg_mod}
\begin{tabular}{|c|c|c|c|c|c|}
\hline
Model& Naive & APS & RAPS & NCP(APS) & NCP(RAPS)\\ \hline
\multicolumn{6}{|c|}{ImageNet($1 - \alpha = 0.900$)} \\ \hline
ResNet152  & 0.937(0.002) & 0.940(0.003) & 0.922(0.003) & 0.905(0.005) & 0.904(0.007) \\ \hline
ResNet101  & 0.936(0.002) & 0.938(0.002) & 0.919(0.003) & 0.903(0.004) & 0.905(0.008) \\ \hline
ResNet50  & 0.934(0.001) & 0.938(0.003) & 0.918(0.004) & 0.902(0.005) & 0.904(0.007) \\ \hline
ResNet18  & 0.925(0.003) & 0.928(0.004) & 0.910(0.003) & 0.901(0.006) & 0.903(0.006) \\ \hline
ResNeXt101 & 0.935(0.001) & 0.939(0.002) & 0.920(0.003) & 0.905(0.003) & 0.903(0.006) \\ \hline
DenseNet161 & 0.934(0.001) & 0.937(0.003) & 0.919(0.003) & 0.902(0.007) & 0.904(0.007) \\ \hline
VGG16  & 0.929(0.003) & 0.930(0.003) & 0.912(0.003) & 0.906(0.004) & 0.904(0.004) \\ \hline
Inception & 0.919(0.002) & 0.927(0.004) & 0.909(0.004) & 0.910(0.003) & 0.905(0.003) \\ \hline
ShuffleNet & 0.927(0.003) & 0.932(0.004) & 0.907(0.002) & 0.911(0.005) & 0.907(0.003) \\ \hline
\multicolumn{6}{|c|}{CIFAR100 ($1 - \alpha = 0.900$)} \\ \hline
ResNet18 & 0.931(0.005) & 0.938(0.004) & 0.925(0.003) & 0.908(0.007) & 0.909(0.006)    \\ \hline
ResNet50 & 0.947(0.006) & 0.944(0.005) & 0.919(0.004) & 0.906(0.008) & 0.906(0.011)\\ \hline
ResNet101 & 0.940(0.005) & 0.944(0.004) & 0.926(0.004) & 0.907(0.008) & 0.907(0.006)\\ \hline 
\multicolumn{6}{|c|}{CIFAR10($1 - \alpha = 0.960$)} \\ \hline
ResNet18 & 0.979(0.002) & 0.982(0.002) & 0.982(0.002) & 0.962(0.005) & 0.961(0.004) \\ \hline
ResNet50 & 0.981(0.003) & 0.984(0.003) & 0.984(0.003) & 0.965(0.003) & 0.963(0.003) \\ \hline
ResNet101 & 0.983(0.002) & 0.986(0.002) & 0.985(0.003) & 0.966(0.002) & 0.964(0.002) \\ \hline
\end{tabular}
\end{table*}

\vspace{1.0ex}

\noindent {\bf \texttt{NCP} vs. Baseline methods for classification.} Table \ref{tab:CIFAR100_size_mod} and Table \ref{tab:CIFAR100_cvg_mod} show respectively the results for efficiency (average prediction set size) and marginal coverage for ImageNet, CIFAR100, and CIFAR10 datasets respectively for classification. We discuss all these results along different dimensions and make the following observations. {\bf 1)} The prediction set size is significantly reduced by \texttt{NCP(APS)} for both ImageNet and CIFAR100, compared to the APS and Naive algorithms. The prototypical results shown in Figure \ref{fig:t_SNE_all} demonstrate that the prediction set size from \texttt{NCP} is directly proportional to the clustering property of the learned input representation. {\bf 2)} \texttt{NCP(RAPS)} reduces the prediction set size over APS, RAPS, and Naive baselines. These results demonstrate that \texttt{NCP} being a wrapper algorithm can make use of a better conformity scoring function (i.e., RAPS) to further reduce the prediction set size. {\bf 3)} The reduction in size of prediction set by \texttt{NCP(APS)} and \texttt{NCP(RAPS)} is achieved by further tightening the gap between actual coverage and the desired coverage $1-\alpha$. {\bf 4)} Naive baseline achieves the desired coverage, but the size of its prediction set is significantly larger than all CP methods.

\vspace{1.0ex}

\noindent {\bf \texttt{NCP} vs. Baseline methods for regression.} Table \ref{Regression_intervals} and Table \ref{Regression_coverage}  show respectively the results for efficiency (average prediction interval length) and the marginal coverage for the eleven regression datasets. We make similar observations as in the classification task. {\bf 1)} The prediction interval length is significantly reduced by \texttt{NCP} for all datasets. {\bf 2)} The reduction in length of the prediction interval by \texttt{NCP} is achieved by further tightening the gap between actual coverage and the desired coverage $1-\alpha$.

\begin{table}[t]
\centering
\caption{\textbf{Mean prediction interval size} on Regression datasets. We present the mean and standard deviation over ten different runs.}
\label{Regression_intervals}
\begin{tabular}{|l|l|l|}
\hline
Dataset     & CP          & NCP      \\ \hline
BlogFeedback   & 44.79(2.78) & 18.97(2.42)  \\ \hline
Gas turbine & 16.69(1.47) & 12.94(1.11)  \\ \hline
Facebook\_1 & 50.96(5.72) & 21.89(7.27)  \\ \hline
Facebook\_2 & 44.12(3.85) & 21.28(6.55)  \\ \hline
Facebook\_3 & 46.98(4.01) & 39.73(9.59)  \\ \hline
Facebook\_4 & 46.78(5.35) & 34.01(11.76) \\ \hline
Facebook\_5 & 45.35(0.65) & 43.73(5.55)  \\ \hline
MEPS\_19    & 54.99(2.64) & 35.63(2.54)  \\ \hline
MEPS\_20    & 48.85(2.22) & 33.85(3.53)  \\ \hline
MEPS\_21    & 61.46(7.49) & 36.01(5.10)  \\ \hline
Concrete    & 59.27(0.23) & 50.42(0.62)  \\ \hline
\end{tabular}
\end{table}

\begin{table}[t]
\centering
\caption{\textbf{Marginal coverage} on Regression task with $1-\alpha=0.90$. We present the mean and standard deviation (negligible) over ten different runs for 11 different data sets.}
\label{Regression_coverage}
\begin{tabular}{|l|l|l|}
\hline
Dataset     & CP    & NCP \\ \hline
BlogFeedback   & 0.952 & 0.909   \\ \hline
Gas turbine & 0.957 & 0.911   \\ \hline
Facebook\_1 & 0.951 & 0.912   \\ \hline
Facebook\_2 & 0.951 & 0.917   \\ \hline
Facebook\_3 & 0.950 & 0.941   \\ \hline
Facebook\_4 & 0.952 & 0.928   \\ \hline
Facebook\_5 & 0.949 & 0.946   \\ \hline
MEPS\_19    & 0.948 & 0.916   \\ \hline
MEPS\_20    & 0.953 & 0.923   \\ \hline
MEPS\_21    & 0.948 & 0.914   \\ \hline
Concrete    & 0.973 & 0.910   \\ \hline
\end{tabular}
\end{table}

\section{Related Work}

Conformal prediction (CP) is a general framework for uncertainty quantification that provides (marginal) coverage guarantees without any assumptions on the underlying data distribution \cite{vovk1999machine,vovk2005algorithmic,shafer2008tutorial}. Various instantiations of the CP framework are studied to produce prediction intervals for regression \cite{Papadopoulos08,vovk2012conditional,vovk2015cross,lei2014distribution,vovk2018cross,lei2018distribution,romano2019conformalized,izbicki2019flexible,guan2019conformal,gupta2022nested,kivaranovic2020adaptive,barber2021predictive,foygel2021limits}, and prediction sets for multi-class classification \cite{lei2013distribution,sadinle2019least,romano2020classification,angelopoulos2021uncertainty} and structured prediction \cite{bates2021distribution}. 

This paper focuses on split conformal prediction \cite{Papadopoulos08,lei2018distribution}, which uses a set of calibration examples to provide uncertainty quantification for any given predictor including deep models. Our problem setup considers CP methods to reduce prediction set size under a given marginal coverage constraint. There is relatively less work on CP methods for classification when compared to regression. The adaptive prediction sets (APS) method \cite{romano2020classification} and its regularized version (referred to as RAPS) \cite{angelopoulos2021uncertainty} design conformity scoring functions to return smaller prediction sets. However, neither APS nor RAPS were analyzed theoretically to characterize the reasons for its effectiveness. Our \texttt{NCP} algorithm is a wrapper approach that is complementary to both APS and RAPS as we demonstrated in experiments.  

Localized CP \cite{guan2021localized} was proposed to improve CP by putting more emphasis on the local neighborhood of testing examples and was only evaluated on synthetic regression tasks.  Recent work \cite{lin2021locally} applied LCP for real-world regression tasks with good experimental results. {\em However, there is no existing theoretical analysis to characterize the precise conditions for reduced prediction intervals or prediction sets.} In fact, this analysis is not possible without defining a concrete localizer function. The proposed neighborhood CP algorithm specifies an effective localizer function by leveraging the input representations learned by deep models, develops theory to characterize when and why neighborhood CP produces smaller prediction sets for classification problems, and performs empirical evaluation on real-world classification datasets using diverse deep models to demonstrate the effectiveness of this simple approach.

There are also other approaches which are not based on conformal prediction to produce prediction sets with small sizes for regression \cite{pearce2018high,chen2021learning} and classification \cite{park2019pac} tasks. Since the focus of this paper is on improving CP based uncertainty quantification, these methods are out of scope for our study.

\section{Summary}

This paper studied a novel neighborhood conformal prediction (\texttt{NCP}) algorithm to improve uncertainty quantification of pre-trained deep classifiers. 
For a given testing input, \texttt{NCP} identifies $k$ nearest neighbors and assigns importance weights proportional to their distance defined using the learned representation of deep classifier. The theoretical analysis characterized why \texttt{NCP} reduces the prediction set size over standard CP framework. Our experimental results corroborated the developed theory and demonstrated significant reduction in prediction set size over prior CP methods on diverse classification benchmarks and deep models. 

\section*{Acknowledgments} This research is supported in part by Proofpoint Inc. and the AgAID AI Institute for Agriculture Decision Support, supported by the National Science Foundation and United States Department of Agriculture - National Institute of Food and Agriculture award \#2021-67021-35344. The authors would like to thank the feedback from anonymous reviewers who provided suggestions to improve the paper.

\small
\bibliography{reference}

\clearpage

\onecolumn


\section{Experimental and Implementation Details}

\begin{table}[!h]
\centering
\caption{The size of different splits for each classification data set.}
\label{tab:data_details}
\begin{tabular}{|c|c|c|c|c|}
\hline
Dataset & Calibration data & Scaling data & Validation data & Testing data\\ 
 & &  & &\\ \hline
ImageNet & 5000 & 5000 & 15000 & 25000\\ \hline
CIFAR100 & 3000 & 1000 & 3000 & 3000\\ \hline
CIFAR10 & 3000 & 1000 & 3000 & 3000\\ \hline

\end{tabular}
\end{table}

\section{Neighborhood Conformal Prediction (NCP) vs. Conformal Prediction (CP)}

\vspace{2.0ex}

\subsection{Proof of Theorem \ref{theorem:improved_LCP_over_CP} }
\label{section:proof_improved_LCP_over_CP}

Before proving Theorem \ref{theorem:improved_LCP_over_CP}, we present the following technical lemma.
\begin{lemma}
\label{lemma:decreased_quantile}
(Increase of quantile in decreasing probability function of threshold $t$)
Suppose $G_0, G_1$ are monotonically non-decreasing functions.
Given $t$, if $G_0(t) \geq G_1(t)$ and $\alpha \in (0, 1)$,
we have
\begin{align*}
\min\{ t : G_0(t) \geq 1 - \alpha \}
\leq
\min\{ t : G_1(t) \geq 1 - \alpha \}.
\end{align*}
\end{lemma}

\begin{proof}
(of Lemma \ref{lemma:decreased_quantile})

\begin{align*}
\min\{ t : G_0(t) \geq 1 - \alpha \}
= &
\min\{ t : G_1(t) \geq 1 - \alpha - \underbrace{ ( G_0(t) - G_1(t) ) }_{ \geq 0 } \}
\\
\leq &
\min\{ t : G_1(t) \geq 1 - \alpha \},
\end{align*}
where the inequality is due to the monotonic increase of $G_1$ for achieving a larger minimum value, from $1-\alpha - (G_0(t)-G_1(t))$ to $1-\alpha$.

\end{proof}

Lemma \ref{lemma:decreased_quantile} is simple and used very frequently in the following analysis. 
The main functionality of Lemma \ref{lemma:decreased_quantile} is to track the increase of quantile (the threshold found by taking $\min$ as shown above), as the decrease of probability function, which is generalized by $G_0(t)$ and $G_1(t)$.

\begin{proof} (Proof of Theorem \ref{theorem:improved_LCP_over_CP}, improved efficiency of NCP over CP)

In this proof, we use the notation narrowed down in Section 4 for the multi-class classification setting.
We begin with an artificial conformal prediction algorithm that is locally built in a class-wise manner, which is named by class-wise NCP and is different from NCP that locally builds quantile on each data sample.
Subsequently, we use this new algorithm as a proxy to build the connection between NCP and CP.
Specifically, we build expected quantile functions of NCP, class-wise NCP, and CP in population, respectively, all of which take the desired significance level $\alpha$ as input.
If we fix the desired significance level for all algorithms and derive the expected quantile values on the entire data distribution, then it is easy to compare and identify under what conditions, NCP can be more efficient (i.e., smaller prediction set/interval size) than CP.

Recall that NCP algorithm aims to find the following $\tilde \alpha$ given $\alpha$:
\begin{align*}
\alpha^\NCP(\alpha)
\triangleq
\max\{ \tilde \alpha : \P_{X} \{ X \leq Q^\NCP(\tilde \alpha; X) \} \geq 1 - \alpha \} ,
\end{align*}
where the quantile for $X$ with significance $\tilde \alpha$ is defined as follows.
\begin{align*}
Q^\NCP(\tilde \alpha; X)
\triangleq
\min\{ t : \underbrace{ \P_{X'} \{ V(X', F^*(X')) \leq t, X' \in \calN_\calB(X) \} }_{ \triangleq G^\NCP(t; X) } \geq 1 - \tilde \alpha \}
\end{align*}
is the quantile of $X$ derived from its neighborhood region $\calN_\calB$.
Note that the above definitions are in population, which is different from the empirical definition in (\ref{eq:LCP_tilde_alpha}).
Above $G^\NCP(t; X)$ is the coverage probability function given a threshold $t$ (like the function in Lemma \ref{lemma:decreased_quantile}).

Given $t$, for $G(t; X)$ and $F^*(X) = c \in [C]$, we have the following development:

\begin{align}\label{eq:probability_LCP_to_classLCP}
G^\NCP(t; X) =
&
\P_{X'} \{ V(X') \leq t, X' \in \calN_\calB(X) \}
\nonumber\\
= &
\P_{X} \{ X \in \calR_\calB^* \cap \calX_c \} \cdot \P_{X'} \{ V(X', F^*(X') ) \leq t, X' \in \calN_\calB(X) | X \in \calR_\calB^* \cap \calX_c \}
\nonumber\\
&
+ \underbrace{
\P_{X} \{ X \notin \calR_\calB^* \cap \calX_c \} \cdot \P_{X'} \{ V(X', F^*(X')) \leq t, X' \in \calN_\calB(X) | X \notin \calR_\calB^* \cap \calX_c \}
}_{ \geq 0 }
\nonumber\\
\stackrel{(a)}{\geq} &
\P_{X} \{ X \in \calR_\calB^* \cap \calX_c \} \cdot \P_{X'} \{ V(X', F^*(X')) \leq t, X' \in \calN_\calB(X) | X \in \calR_\calB^* \cap \calX_c \}
\nonumber\\
\stackrel{(b)}{\geq} &
( 1 - \mu_\calB ) \cdot \P_{X'} \{ V(X', F^*(X')) \leq t, X' \in \calN_\calB(X) | X \in \calR_\calB^* \cap \calX_c \}
\nonumber\\
\stackrel{(c)}{\geq} &
\underbrace{ ( 1 - \mu_\calB ) \cdot \sigma }_{ \triangleq \hat \sigma } \cdot \underbrace{ \P_{X'} \{ V(X', F^*(X')) \leq t | F^*(X') = c \} }_{ \triangleq G^\NCP_\class(t; c) },
\end{align}

where the above inequality $(a)$ is due to $\P\{ \cdot \} \geq 0$ for any event.
Inequality $(b)$ is due to $\mu_\calB$-separation of $\calD_\calX$ assumed in Assumption \ref{assumption:LCP}.
Inequality $(c)$ is due to $\sigma$-concentration of any quantile on $\calR_\calB^*$ assumed in Assumption \ref{assumption:LCP}.

As mentioned in Section 4, this concentration condition shows that data samples satisfying the certain quantile $t$ (i.e., for $X$ such that $V(X, F^*(X)) \leq t$) are significantly densely distributed on the robust set $\calR_\calB^*$ (or its class-wise partition $\calR_\calB^* \cap \calX_c$ for the class $c$).
This significance of dense distribution is captured by the condition number $\sigma$, which is explicitly shown as follows
\begin{align*}
\sigma 
\leq &
\frac{ \P_{X'} \{ V(X', F^*(X')) \leq t, X' \in \calN_\calB(X) | X \in \calR_\calB^* \cap \calX_c \} }{ \P_{X'} \{ V(X', F^*(X')) \leq t | F^*(X') = c \} }.
\end{align*}

Now we can design an artificial conformal prediction algorithm, i.e., {\it class-wise NCP}, to solve the following problem:
\begin{align*}
\alpha^\NCP_\class(\alpha)
\triangleq &
\min\{ \tilde \alpha : \P_{X} \{ X \leq Q_\class^\NCP(\tilde \alpha; F^*(X)) \} \geq 1 - \alpha \},
\end{align*}
where the {\it class-wise quantile} $Q_\class^\NCP(\tilde\alpha; c)$ for the class $c$ can be determined as follows
\begin{align}\label{eq:quantile_classLCP}
Q^\NCP_\class(\tilde \alpha; c)
\triangleq
\min\{ t : \underbrace{ \P_{X'}\{ V(X', F^*(X')) \leq t | F^*(X') = c \} }_{ = G^\NCP_\class(t; c) } \geq 1 - \tilde \alpha \}.
\end{align}

In the following three steps, we use this artificial class-wise NCP algorithm as a proxy to build connection between NCP and CP, respectively, from which we show the condition that makes NCP more efficient than CP.

{\bf Step 1. Connection between NCP and class-wise NCP}

In the first step, we show 
from inequalities in (\ref{eq:probability_LCP_to_classLCP}), we have
\begin{align*}
G^\NCP(t; X)
\geq
\underbrace{ (1-\mu_\calB) \sigma }_{ = \hat \sigma } \cdot G_\class^\NCP(t; F^*(X))
\end{align*}

Therefore, by Lemma \ref{lemma:decreased_quantile}, we have the following relation between the quantiles of NCP and class-wise NCP:
\begin{align}\label{eq:quantile_LCP_to_classLCP}
Q^\NCP(\tilde \alpha; X)
= &
\min\{ t : G(t; X) \geq 1 - \tilde \alpha \}
\nonumber\\
\leq &
\min\{ t : \hat \sigma \cdot G_\class^\NCP(t; Y) \geq 1 - \tilde \alpha \}
\nonumber\\
= &
\min\{ t : G_\class^\NCP(t; Y) \geq \frac{ 1 - \tilde \alpha }{ \hat \sigma } \}
\nonumber\\
= &
Q_\class^\NCP \Big( \frac{ \hat \sigma - 1 + \tilde \alpha }{ \hat \sigma } ; F^*(X) \Big) ,
\end{align}
where the last equality is due to 
\begin{align*}
\frac{ 1 - \tilde \alpha }{ \hat \sigma }
=
\frac{ \hat \sigma  - \hat \sigma + 1 - \tilde \alpha }{ \hat \sigma }
=
1 - \frac{ \hat \sigma - 1 + \tilde \alpha }{ \hat \sigma }.
\end{align*}

Note that 
\begin{align*}
\frac{ \hat \sigma - 1 + \tilde \alpha }{ \hat \sigma } - \tilde \alpha
= 
\frac{\hat \sigma - 1}{\hat \sigma} + \tilde \alpha \Big( \frac{1}{\hat \sigma} - 1 \Big)
=
1 - \frac{1}{\hat \sigma} + \tilde \alpha \Big( \frac{1}{\hat \sigma} - 1 \Big)
=
\Big( 1 - \frac{1}{\hat \sigma} \Big) ( 1 - \tilde \alpha )
\geq 0,
\end{align*}
so we have $\frac{ \hat \sigma - 1 + \tilde \alpha }{ \hat \sigma } \geq \tilde \alpha$.
This shows that, to achieve more certain prediction, i.e., smaller significance level $\tilde \alpha$, NCP requires smaller quantile than class-wise NCP, as shown in (\ref{eq:quantile_LCP_to_classLCP}).

Before we finish this step, we can have a further investigation of the expected quantile determined by NCP and its link to class-wise NCP in (\ref{eq:quantile_LCP_to_classLCP}) as follows.
\begin{align}\label{eq:class_conditional_quantile_LCP_to_classLCP}
\bar Q^\NCP(\tilde \alpha)
\triangleq
\sum_{c=1}^C \P\{ X \in \calX_c \} \cdot Q^\NCP(\tilde \alpha; X | X \in \calX_c )
\nonumber\\
\leq 
\sum_{c=1}^C \P\{ X \in \calX_c \} \cdot Q^\NCP_\class\Big( \frac{ \hat\sigma - 1 + \tilde \alpha }{ \hat \sigma } ; c \Big) .
\end{align}

{\bf Step 2. Connection between class-wise NCP and CP}

In this step, we show that $\alpha^\NCP_\class(\alpha) \leq \alpha$, which implies that to achieve the same significance level with CP, class-wise NCP can simply find the quantiles using $\alpha$ directly on each class.
To show this equality, we have
\begin{align*}
\alpha^\NCP_\class(\alpha)
= &
\max\{ \tilde \alpha : \sum_{c=1}^C \P_{X}\{ F^*(X) = c \} \cdot \P_{ X }\{ X \leq Q_\class^\NCP(\tilde \alpha; c) | F^*(X) = c\} \geq 1 - \alpha \}
\\
\leq &
\max\{ \tilde \alpha : \sum_{c=1}^C \P_{X}\{ F^*(X) = c \} \cdot ( 1 - \tilde \alpha ) \geq 1 - \alpha \}
\\
= &
\max\{ ( \tilde \alpha : 1 \cdot ( 1 - \tilde \alpha ) \geq 1 - \alpha \}
= \alpha .
\end{align*}
The above result states that to achieve $\alpha$ mis-coverage probability, class-wise NCP requires to find the quantile to give $\alpha$ mis-coverage for each of the classes, by which the final coverage probability can be lower bounded by $1-\alpha$.
However, only deriving the above significance level is not sufficient to investigate the efficiency of class-wise NCP.
Below, we try to lower bound the coverage probability function of class-wise NCP with CP quantile replacing class-wise NCP quantile.

We now use a somewhat similar idea as in (\ref{eq:probability_LCP_to_classLCP}), where we derive the coverage probability from NCP by using the presentation of class-wise NCP.
We make use of the assumptions of the robust set on the data sample distribution and the alignment between the quantile of this distribution property.
To build connection from class-wise NCP to CP, we now treat classes in a similar way: we group all classes into two categories, i.e., the robust class set and non-robust class set.

Specifically, let $R_\class(\tilde \alpha) = \{ c \in [C] : Q_\class^\NCP(\tilde \alpha; c) \leq Q^\CP(\tilde \alpha) \}$ denote the robust set of class labels, for which the quantile determined by class-wise NCP is smaller than that determined by CP given the same $\tilde \alpha$.
As a result, we can verify that
\begin{align*}
&
\P_X \{ V(X, F^*(X)) \leq Q_\class^\NCP(\tilde \alpha; F^*(X)) | F^*(X) \notin R_\class \}
\\
\geq &
\P_X \{ V(X, F^*(X)) \leq Q^\CP(\tilde \alpha) | F^*(X) \notin R_\class \}
\geq 
1 - \tilde \alpha .
\end{align*}
For the other group of classes, we can consider the worst-case, i.e., the minimum class distribution probability $P_\class^{min}$, to bound.
It is easy to see that $| R_\class(\tilde \alpha)| \geq 1$.

We now provide details of the relation of actual coverage probability for using different quantiles determined by class-wise NCP and CP with the same significance level $\tilde \alpha$.
For the class conditional coverage probability for class $c$, we can plug $Q_\class^\NCP(\tilde \alpha; c)$ in (\ref{eq:quantile_classLCP}) into $G_\class^\NCP(t; c)$ and replace $t$ as follows.
\begin{align*}
&
G_\class^\NCP( Q_\class^\NCP(\tilde \alpha; c) ; c )
\\
= &
\P_X \{ V(X, F^*(X)) \leq Q_\class^\NCP(\tilde \alpha; c) | F^*(X) = c \}
\\
= &
\P\{ c \in R_\class(\tilde \alpha) \} \cdot \P_X \{ V(X, F^*(X)) \leq Q_\class^\NCP(\tilde \alpha; c) | F^*(X) = c \in R_\class(\tilde \alpha) \}
\\
&
+ \P\{ c \notin R_\class(\tilde \alpha) \} \cdot \P_X \{ V(X, F^*(X)) \leq Q_\class^\NCP(\tilde \alpha; c) | F^*(X) = c \notin R_\class(\tilde \alpha) \}
\\
\stackrel{(a)}{\geq} & 
\P\{ c \in R_\class(\tilde \alpha) \} \cdot \P_X \{ V(X, F^*(X)) \leq Q_\class^\NCP(\tilde \alpha; c) | F^*(X) = c \in R_\class(\tilde \alpha) \}
\\
&
+ \P\{ c \notin R_\class(\tilde \alpha) \} \cdot \P_X \{ V(X, F^*(X)) \leq Q^\CP(\tilde \alpha) | F^*(X) = c \notin R_\class(\tilde \alpha) \}
\\
= &
\P_X \{ V(X, F^*(X)) \leq Q^\CP(\tilde \alpha) | F^*(X) = c \}
\\
&
+ \P\{ c \in R_\class(\tilde \alpha) \} \cdot 
\underbrace{
\P_X \{ V(X, F^*(X)) \leq Q_\class^\NCP(\tilde \alpha; c) | F^*(X) = c \in R_\class(\tilde \alpha) \}
}_{ \geq 1-\tilde \alpha }
\\
&
- \P\{ c \in R_\class(\tilde \alpha) \} \cdot \P_X \{ V(X, F^*(X)) \leq Q^\CP(\tilde \alpha) | F^*(X) = c \in R_\class(\tilde \alpha) \}
\\
\stackrel{(b)}{\geq} &
\P_X \{ V(X, F^*(X)) \leq Q^\CP(\tilde \alpha) | F^*(X) = c \}
\\
&
+ \P\{ c \in R_\class(\tilde \alpha) \} \cdot (1 - \tilde \alpha)
\\
&
- \P\{ c \in R_\class(\tilde \alpha) \} \cdot \P_X \{ V(X, F^*(X)) \leq Q^\CP(\tilde \alpha) | F^*(X) = c \in R_\class(\tilde \alpha) \}
\\
= &
\P_X \{ V(X, F^*(X)) \leq Q^\CP(\tilde \alpha) | F^*(X) = c \}
\\
&
+ \P\{ c \in R_\class(\tilde \alpha) \} \cdot \Big( 1 - \tilde \alpha 
- \underbrace{ 
\P_X \{ V(X, F^*(X)) \leq Q^\CP(\tilde \alpha) | F^*(X) = c \in R_\class(\tilde \alpha) \}
}_{ \leq 1 }
\Big)
\\
\stackrel{(c)}{\geq} &
\P_X \{ V(X, F^*(X)) \leq Q^\CP(\tilde \alpha) | F^*(X) = c \}
- \tilde \alpha ( 1 - P_\class^{min} ) 
,
\end{align*}
where the first inequality $(a)$ is due to the definition of $R_\class(\tilde \alpha)$, i.e., we can replace $Q_\class^\NCP(\tilde \alpha; c)$ with $Q^\CP(\tilde \alpha)$ for $c \notin R_\class(\tilde \alpha)$.
The second inequality $(b)$ is due to 
$$\P_X\{ V(X, F^*(X)) \leq Q_\class^\NCP(\tilde \alpha; c) | F^*(X) = c \in R_\class(\tilde \alpha) \} \geq 1 - \tilde \alpha.$$
The last inequality $(c)$ is due to 
$$\P_X \{ V(X, F^*(X)) \leq Q^\CP(\tilde \alpha) | F^*(X) = c \in R_\class(\tilde \alpha) \} \leq 1 ,$$ 
and note that if $| R_\class | \neq C$ (at least one class not in $R_\class$), then
\begin{align*}
&
\P\{ c \notin R_\class(\tilde \alpha) \} \geq P_\class^{min}
\\
\Rightarrow & \\
&
\P\{ c \in R_\class(\tilde \alpha) \} \leq 1 -  P_\class^{min} .
\end{align*}

From above, after inequality $(a)$, we can view it as the class-conditional coverage of a new conformal prediction algorithm, i.e., using $Q^\NCP_\class(\tilde \alpha; c)$ for $c \in R_\class(\tilde \alpha)$ and using $Q^\CP(\tilde \alpha)$ for $c \notin R_\class(\tilde \alpha)$, which definitely has smaller expected quantile than CP across all classes.

Then we have the following full coverage probability using class conditional ones:
\begin{align*}
&
\sum_{c=1}^C \P\{ X \in \calX_c \} \cdot \underbrace{ \P_X\{ V(X, F^*(X)) \leq Q_\class^\NCP(\tilde \alpha; c) \} | F^*(X) = c }_{ = G_\class^\NCP(\tilde \alpha; c) } \}
\\
\geq &
\sum_{c=1}^C \P\{ X \in \calX_c \} \cdot \P_X\{ V(X, F^*(X)) \leq Q^\CP(\tilde \alpha) | F^*(X) = c \} 
- \tilde \alpha ( 1 - P_\class^{min} )
\\
= &
\P_X\{ V(X, F^*(X)) \leq Q^\CP(\tilde \alpha) \}
- \tilde \alpha (1 - P_\class^{min} )
\geq 
1 - \tilde \alpha - \tilde \alpha ( 1 - P_\class^{min} )
\\
= &
1 - \tilde \alpha ( 2 - P_\class^{min} ) .
\end{align*}
The above inequality shows that if using a constant quantile (determined by CP using target $\tilde \alpha$ significance level) in class-wise NCP setting, we only achieve $1 - \tilde \alpha ( 2 - P_\class^{min} )$ coverage.
To achieve the same level of $1 - \tilde \alpha$ as CP using this constant quantile, then we need to reduce the mis-coverage parameter to $\tilde \alpha / ( 2 - P_\class^{min})$ from $\tilde \alpha$.

Then we can compute the upper bound of the expected quantile of class-wise NCP as follows:
\begin{align}\label{eq:quantile_classLCP_to_CP}
\bar Q_\class^\NCP(\tilde \alpha)
\triangleq
\sum_{c=1}^C \P\{ X \in \calX_c \} \cdot Q_\class^\NCP(\tilde \alpha; c)
\leq 
\sum_{c=1}^C \P\{ X \in \calX_c \} \cdot Q^\CP (\tilde \alpha / ( 2 - P_\class^{min} ) )
=
Q^\CP(\tilde \alpha / ( 2 - P_\class^{min} ) ) .
\end{align}

{\bf Step 3: Connection between the quantiles determined by NCP and CP via class-wise NCP}

Now we have by (\ref{eq:quantile_classLCP_to_CP}):
\begin{align*}
\bar Q_\class^\NCP(\tilde \alpha)
\leq 
Q^\CP( \tilde \alpha / ( 2 - P_\class^{min} ) ),
\end{align*}
and by (\ref{eq:quantile_LCP_to_classLCP}) and (\ref{eq:class_conditional_quantile_LCP_to_classLCP}):
\begin{align*}
\bar Q^\NCP(\tilde \alpha)
\leq 
\sum_{c=1}^C \P\{ X \in \calX_c \} \cdot Q^\NCP_\class\Big( \frac{ \hat \sigma - 1 + \tilde \alpha }{ \hat \sigma } ; c \Big)
=
\bar Q_\class^\NCP\Big( \frac{ \hat \sigma - 1 + \tilde \alpha }{ \hat \sigma } \Big) .
\end{align*}

If we would like to make NCP more efficient (i.e., smaller predicted set/interval size) than CP, then the following inequality is required to hold:
\begin{align*}
\bar Q^\NCP(\tilde \alpha)
\leq 
\bar Q_\class^\NCP\Big( \frac{ \hat \sigma - 1 + \tilde \alpha }{ \hat \sigma } \Big)
\leq 
Q^\CP \Big( \frac{ \hat \sigma - 1 + \tilde \alpha }{ \hat \sigma ( 2 - P_\class^{min} ) } \Big),
\end{align*}
or equivalently by letting $\alpha = \frac{ \hat \sigma - 1 + \tilde \alpha }{ \hat \sigma ( 2 - P_\class^{min} ) }$:
\begin{align*}
\bar Q^\NCP( 1 - \hat \sigma ( 1 - ( 2 - P^{min}_\class ) \alpha ) )
\leq 
Q^\CP( \alpha ) .
\end{align*}

It suffices to make sure the mis-coverage has the following relation:
\begin{align*}
& 1 - \hat \sigma ( 1 - ( 2 - P^{min}_\class ) \alpha
\leq 
\alpha
\\
\Leftrightarrow & \\
&
1 - \alpha 
\leq 
\hat \sigma ( 1 - ( 2 - P^{min}_\class) \alpha )
\\
\Leftrightarrow & ~~(\text{due to } \alpha \leq 1 / 2 \text{ according to Assumption \ref{assumption:LCP}}) \\
&
\hat \sigma
\geq 
\frac{ 1 - \alpha }{ 1 - ( 2 - P_\class^{min}) \alpha } ,
\end{align*}
where the last inequality holds under Assumption \ref{assumption:LCP}.

\end{proof}

\subsection{Additional Experiments}

\begin{table}[!h]
\centering
\caption{Clustering coefficient for ImageNet-val data. Silhouette score can have values between -1 to 1. Higher silhouette score means better clustering property. These results justify the use of the learned representation to define the neighborhood and weighting function for NCP algorithm.}
\label{tab:cvg}

\begin{tabular}{|c|c|c|c|c|}

    \hline
     Models & silhouette score(raw data) & silhouette score(representation)\\
     \hline
     ResNetXt101&-0.379&0.076 \\
     \hline
     ResNet152&-0.510&0.052\\
     \hline
     ResNet101&-0.477&0.045\\
     \hline
     ResNet50&-0.444&0.034\\
     \hline
     ResNet18&-0.372&0.001\\
     \hline
     DenseNet161&-0.336&0.021\\
     \hline
     VGG16&-0.183&0.012\\
     \hline
     Inception&-0.483&0.060\\
     \hline
     ShuffleNet&-0.346&-0.003\\
     \hline
\end{tabular}

\end{table}

\begin{table}[!h]
\centering
\caption{Ablation results comparing NCP and NCP variant using all calibration examples as neighborhood (NCP-All). Both NCP and NCP-All use the conformity scoring function of APS. NCP reduces the mean prediction set size as shown in table \ref{tab:LCP_NCP} by using a relatively smaller neighborhood. Both NCP and NCP-All achieve $1 - \alpha$ coverage.}
\label{tab:LCP_NCP_cvg}
\begin{tabular}{|c|cc|cc|cc|}
\hline
\multicolumn{1}{|c|}{Dataset}   & \multicolumn{2}{c|}{ImageNet}             & \multicolumn{2}{c|}{CIFAR100}             & \multicolumn{2}{c|}{CIFAR10}              \\ \hline
             $1 - \alpha$                   & \multicolumn{2}{c|}{$0.900$} & \multicolumn{2}{c|}{$0.900$} & \multicolumn{2}{c|}{$ 0.960$} \\ \hline
\multicolumn{1}{|c|}{Models}     & \multicolumn{1}{c|}{NCP-All}       & NCP      & \multicolumn{1}{c|}{NCP-All}       & NCP      & \multicolumn{1}{c|}{NCP-All}       & NCP      \\ \hline
\multicolumn{1}{|c|}{ResNet18}  & \multicolumn{1}{c|}{0.905}     & 0.901    & \multicolumn{1}{c|}{0.928}     & 0.908    & \multicolumn{1}{c|}{0.980}     & 0.962    \\ \hline
\multicolumn{1}{|c|}{ResNet50}  & \multicolumn{1}{c|}{0.921}     & 0.902    & \multicolumn{1}{c|}{0.939}     & 0.906    & \multicolumn{1}{c|}{0.982}     & 0.965    \\ \hline
\multicolumn{1}{|c|}{ResNet101} & \multicolumn{1}{c|}{0.920}     & 0.903    & \multicolumn{1}{c|}{0.936}     & 0.907    & \multicolumn{1}{c|}{0.983}     & 0.966    \\ \hline
\end{tabular}
\end{table}

\begin{table}[!h]
\centering
\caption{The average percentage of calibration examples used by \textbf{NCP(APS)} as nearest neighbors  to produce prediction sets. RN denotes ResNet.}
\label{tab:K_fraction}
\begin{tabular}{|c|c|c|c|c|c|c|c|c|c|}

\hline
Model->     & RN152 & RN101 &  RN50 & RN18 & ResNeXt101 & DenseNet161 & VGG16 & Inception & ShuffleNet\\ \hline
ImageNet  &  28.0 & 33.2 & 12.8 & 24.8 & 30.0 & 15.6 & 46.0 & 30.0 & 24.0\\ \hline
CIFAR100  & - & 12.0 & 08.3 & 12.3 & -  & - & - & - & -  \\ \hline
CIFAR10 &  - & 13.0 & 15.0 & 09.0 & - & - & - & - & -  \\ \hline
\end{tabular}
\end{table}

\begin{figure}[!h]
    \centering
    \includegraphics[scale = 0.4]{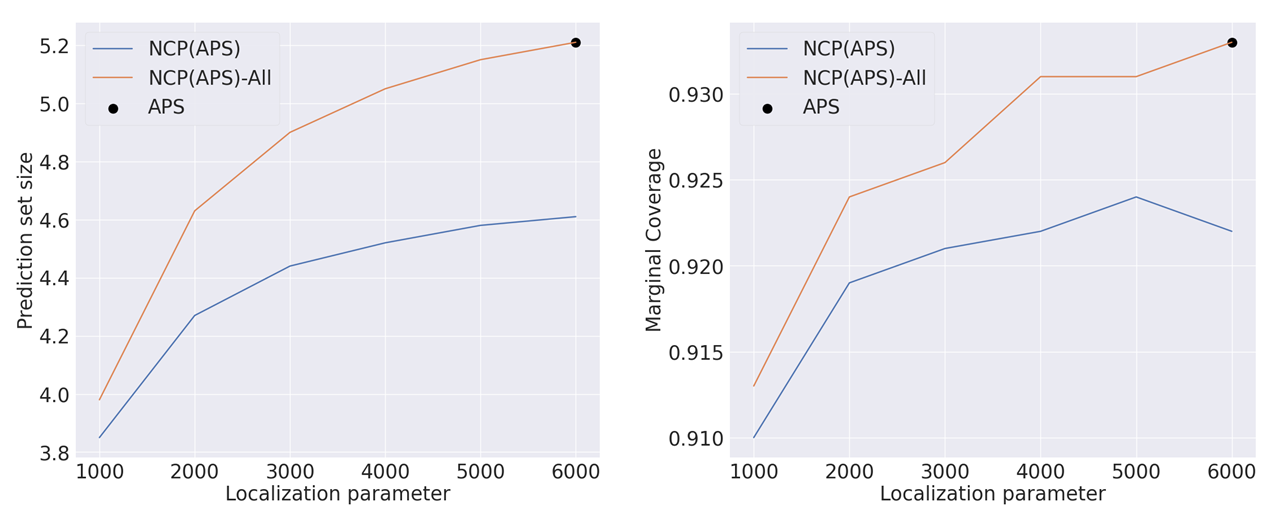}
    \caption{Results showing coverage and prediction set size of d prediction set size of APS, NCP(APS), and NCP(APS)-All for classification as a function of the localization parameter for ResNet18 model on CIFAR100 dataset. A very high value of localization parameter degenerates to APS.}
    \label{fig:classi_loc_RN18}
\end{figure}

\begin{figure}[!h]
    \centering
    \includegraphics[scale = 0.4]{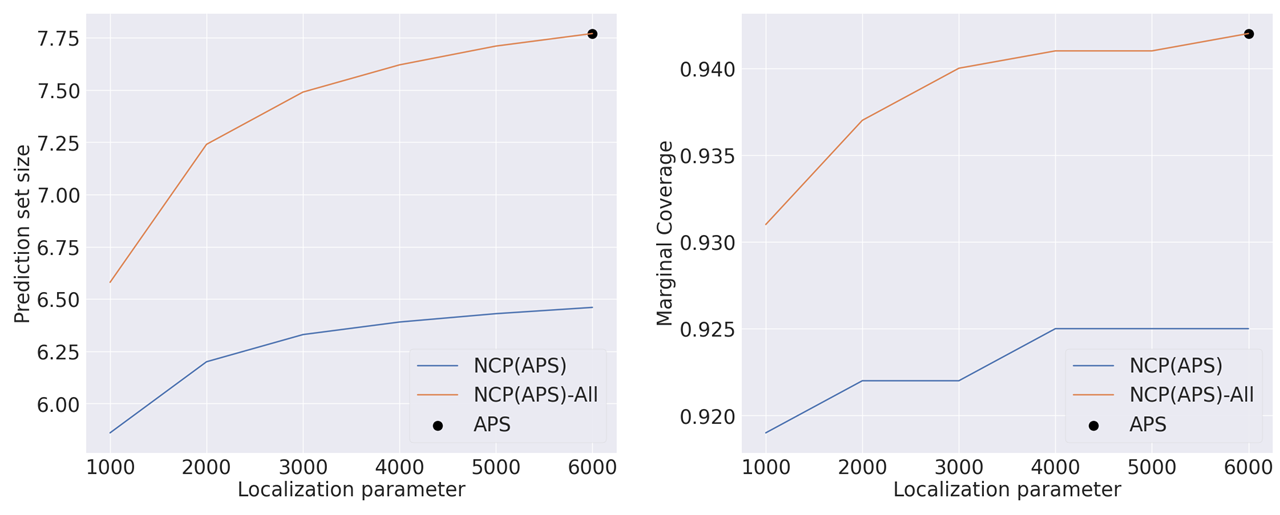}
    \caption{Results showing coverage and prediction set size of d prediction set size of APS, NCP(APS), and NCP(APS)-All for classification as a function of the localization parameter for ResNet50 model on CIFAR100 dataset. A very high value of localization parameter degenerates to APS.}
    \label{fig:classi_loc_RN50}
\end{figure}

\begin{figure}[!h]
    \centering
    \includegraphics[scale = 0.4]{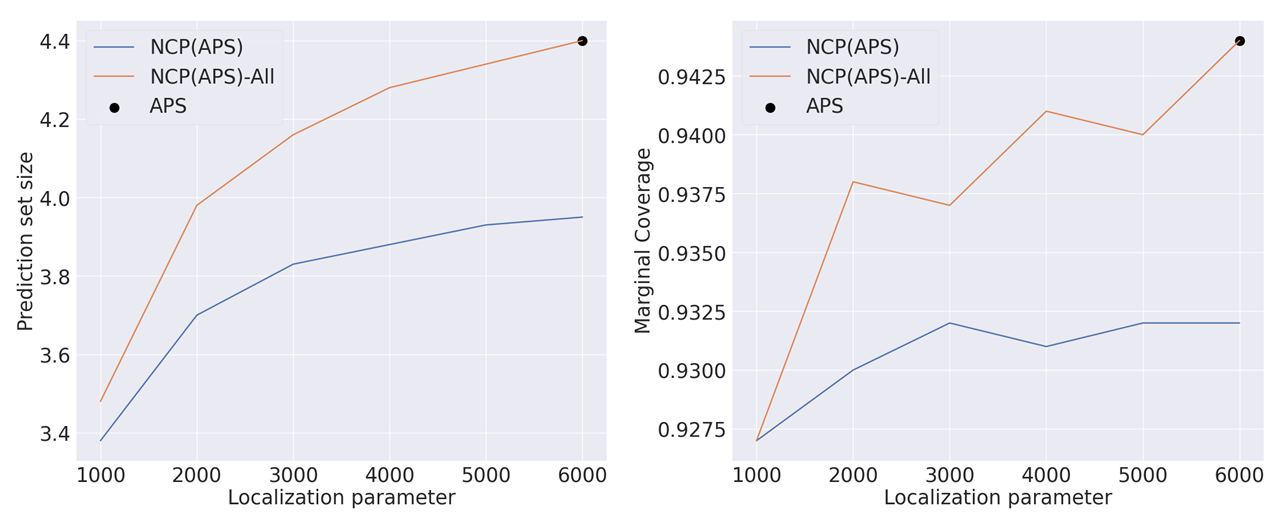}
    \caption{Results showing coverage and prediction set size of APS, NCP(APS), and NCP(APS)-All for classification as a function of the localization parameter for ResNet101 model on CIFAR100 dataset. A very high value of localization parameter degenerates to APS.}
    \label{fig:classi_loc_RN101}
\end{figure}

\begin{table}[]
\centering
\begin{tabular}{|ccccc|}
\hline
\multicolumn{1}{|c|}{Model}       & \multicolumn{1}{c|}{Inference} & \multicolumn{1}{c|}{RAPS} & \multicolumn{1}{c|}{NCP(RAPS)} & NCP(RAPS) with LSH\\ \hline
\multicolumn{5}{|c|}{ImageNet}                                                                                                                                                      \\ \hline
\multicolumn{1}{|c|}{ResNet152}   & \multicolumn{1}{c|}{1.215}        & \multicolumn{1}{c|}{1.043}        & \multicolumn{1}{c|}{1.154}             & 1.046                          \\ \hline
\multicolumn{1}{|c|}{ResNet101}   & \multicolumn{1}{c|}{1.106}        & \multicolumn{1}{c|}{1.022}        & \multicolumn{1}{c|}{1.188}             & 1.037                          \\ \hline
\multicolumn{1}{|c|}{ResNet50}    & \multicolumn{1}{c|}{1.013}        & \multicolumn{1}{c|}{1.062}        & \multicolumn{1}{c|}{1.162}             & 1.071                          \\ \hline
\multicolumn{1}{|c|}{ResNet18}    & \multicolumn{1}{c|}{0.963}        & \multicolumn{1}{c|}{1.019}        & \multicolumn{1}{c|}{1.171}             & 1.028                          \\ \hline
\multicolumn{1}{|c|}{ResNeXt101}  & \multicolumn{1}{c|}{1.748}        & \multicolumn{1}{c|}{1.070}        & \multicolumn{1}{c|}{1.121}             & 1.071                          \\ \hline
\multicolumn{1}{|c|}{DenseNet161} & \multicolumn{1}{c|}{1.690}        & \multicolumn{1}{c|}{1.030}        & \multicolumn{1}{c|}{1.065}             & 1.044                          \\ \hline
\multicolumn{1}{|c|}{VGG16}       & \multicolumn{1}{c|}{1.109}        & \multicolumn{1}{c|}{1.027}        & \multicolumn{1}{c|}{1.153}             & 1.046                          \\ \hline
\multicolumn{1}{|c|}{Inception}   & \multicolumn{1}{c|}{1.434}        & \multicolumn{1}{c|}{1.050}        & \multicolumn{1}{c|}{1.178}             & 1.051                          \\ \hline
\multicolumn{1}{|c|}{ShuffleNet}  & \multicolumn{1}{c|}{0.914}        & \multicolumn{1}{c|}{1.042}        & \multicolumn{1}{c|}{1.101}             & 1.043                          \\ \hline
\multicolumn{5}{|c|}{CIFAR100}                                                                                                                                                      \\ \hline
\multicolumn{1}{|c|}{ResNet18}    & \multicolumn{1}{c|}{0.039}        & \multicolumn{1}{c|}{0.331}        & \multicolumn{1}{c|}{0.388}             & 0.332                          \\ \hline
\multicolumn{1}{|c|}{ResNet50}    & \multicolumn{1}{c|}{0.143}        & \multicolumn{1}{c|}{0.333}        & \multicolumn{1}{c|}{0.367}             & 0.336                          \\ \hline
\multicolumn{1}{|c|}{ResNet101}   & \multicolumn{1}{c|}{0.248}        & \multicolumn{1}{c|}{0.330}        & \multicolumn{1}{c|}{0.361}             & 0.331                          \\ \hline
\multicolumn{5}{|c|}{CIFAR10}                                                                                                                                                       \\ \hline
\multicolumn{1}{|c|}{ResNet18}    & \multicolumn{1}{c|}{0.049}        & \multicolumn{1}{c|}{0.291}        & \multicolumn{1}{c|}{0.311}             & 0.292                          \\ \hline
\multicolumn{1}{|c|}{ResNet50}    & \multicolumn{1}{c|}{0.119}        & \multicolumn{1}{c|}{0.295}        & \multicolumn{1}{c|}{0.311}             & 0.298                          \\ \hline
\multicolumn{1}{|c|}{ResNet101}   & \multicolumn{1}{c|}{0.162}        & \multicolumn{1}{c|}{0.298}        & \multicolumn{1}{c|}{0.313}             & 0.311                          \\ \hline
\end{tabular}
\caption{Runtime comparison (in seconds) during testing time for using different settings. The inference time is shown to provide a frame of reference for the additional CP runtime overhead. We show two variants of NCP algorithm: NCP with standard KNN search algorithm and NCP using Locality Sensitive Hashing (LSH) based KNN search.}
\label{Time_comparison}
\end{table}

\clearpage

\end{document}